\documentclass{article}

\usepackage[nonatbib,preprint]{neurips_2022}

\usepackage[utf8]{inputenc} 
\usepackage[T1]{fontenc}    
\usepackage{url}            
\usepackage{booktabs}       
\usepackage{amsfonts}       
\usepackage{nicefrac}       
\usepackage{microtype}      
\usepackage{amsthm}
\usepackage{amssymb}
\usepackage{color}
\usepackage[algo2e,ruled]{algorithm2e}
\usepackage[normalem]{ulem}
\usepackage{times}

\usepackage[
    backend=biber,
    citestyle=numeric,
	defernumbers=false,       
	giveninits=true,         
	doi=false,
	natbib=true,
	isbn=false,
    maxnames=2, 
	eprint=true,
	url=false,
]{biblatex}

\usepackage[usenames,dvipsnames]{xcolor}
\definecolor{shadecolor}{gray}{0.9}

\usepackage{amsmath}
\usepackage{amsmath,amsthm, mathtools}

\usepackage{hyperref} 
\hypersetup{colorlinks}
\hypersetup{citecolor=MidnightBlue}
\hypersetup{linkcolor=BrickRed}
\hypersetup{urlcolor=MidnightBlue}
\usepackage[nameinlink]{cleveref}
\creflabelformat{equation}{#2\textup{#1}#3}  
\crefname{section}{§}{§§}
\Crefname{section}{§}{§§}

\def\E{\mathrm{E}}
\def\hE{\hat{\E}}

\def\thetavec{{\mathbb{\theta}}}
\def\eps{\epsilon}
\def\reals{\mathbb{R}}

\def\cS{\mathcal{S}}
\def\EOS{\mathrm{EOS}}

\newtheorem{theorem}{Theorem}
\newtheorem{lemma}{Lemma}
\newtheorem{observation}{Observation}

\usepackage{mathtools}
\DeclarePairedDelimiterX{\infdivx}[2]{(}{)}{%
  #1\;\delimsize\|\;#2%
}
\newcommand{\kld}{D_{\text{KL}}\infdivx}

\title{Why GANs are overkill for NLP}

\author{%
  David Alvarez-Melis\thanks{equal contribution, alphabetic ordering}  \\
  Microsoft Research\\
  \texttt{alvarez.melis@microsoft.com} \\
  \And
  Vikas Garg$^*$  \\
  Aalto University; FCAI; YaiYai Ltd\\
  \texttt{vikas.garg@aalto.fi; vikas@yaiyai.fi} \\  
  \And 
  Adam Tauman Kalai$^*$\\
  Microsoft Research \\
  \texttt{adam@kal.ai}
}

\begin{document}

\maketitle

\begin{abstract}%
This work offers a novel theoretical perspective on why, despite numerous attempts, adversarial approaches to generative modeling (e.g., GANs) have not been as popular for certain generation tasks, particularly sequential tasks such as Natural Language Generation, as they have in others, such as Computer Vision. In particular, on sequential data such as text, maximum-likelihood approaches are significantly more utilized than GANs. We show that, while it may seem that maximizing likelihood is inherently different than minimizing distinguishability, this distinction is largely artificial and only holds for limited models. We argue that minimizing KL-divergence (i.e., maximizing likelihood) is a more efficient approach to effectively minimizing the same distinguishability criteria that adversarial models seek to optimize. Reductions show that minimizing distinguishability can be seen as simply boosting likelihood for certain families of models including n-gram models and neural networks with a softmax output layer. To achieve a full polynomial-time reduction, a novel next-token distinguishability model is considered.
\end{abstract}


\section{Introduction}
\vspace{-0.2cm}
Consider a situation where one has samples from a true distribution $p$ over a set $X$ and one wishes to learn to generate similar samples, such as learning to generate English sentences from a large English text corpus. One seeks an approximation $q$ of $p$ which is ``close'' in some sense and from which samples can efficiently be generated.
A common approach to fit these models is Maximum Likelihood Estimation (MLE), which given a training set from $p$ and a parametrized distribution $q_{\theta}$ seeks parameters $\theta$ that maximize the likelihood $q_{\theta}$ assigns to a training set. MLE has long been one of the most popular methods for fitting generative models of sequential data, such as language, where autoregressive neural language models generate remarkably realistic text, e.g., GPT-3 \citep{brown2020language} and PaLM \citep{PaLM2022}. MLE generally involves computing likelihoods $q_\theta(x)$ which can be more challenging in some domains than others, e.g., it may be
more difficult to estimate the probability of a (high-dimensional, real-valued) image than a (discrete-valued) sentence.

An alternative approach, Generative Adversarial Networks (GANs), has become popular across several domains, particularly Computer Vision, owing to breakthrough realism in the images they output \citep[e.g.,][]{G2014, ZGMO2019}. GANs employ an adversarial approach to generation through a zero-sum game between a generator and a distinguisher in which the generator produces samples $x \in X$ which the distinguisher tries to distinguish from real samples from $p$. Often, both the generator and the distinguisher are differentiable neural networks, though this min-max approach of choosing a model whose outputs are nearly indistinguishable from real examples might be considered for any families of generative models and distinguishers. A major advantage of GANs (particularly for images) is that they can be used for generation without requiring computing likelihoods. This advantage is not significant for many sequential models such as language models, where computing likelihoods is not difficult.

In contrast, the adversarial approach has yet to demonstrate significant improvements in some other domains such as Natural Language Processing (NLP). One  well-known barrier to NLP GANs is that language models produce discrete outputs (words), so they are not naturally differentiable \citep{goodfellow2017nips}. However, despite numerous works circumventing this limitation and adapting GANs to text generation \citep{yu2017seqgan, press2017language, guo2018long, dautume2019training}, adversarial-based models have yet to achieve the same popularity or performance gains that were seen for images. In particular, language GANs have been shown to under-perform MLE in terms of quality \citep{tevet2019evaluating} while facing the challenge of lack of diversity due to mode collapse \citep{caccia2018language}, which is a well-known issue with GANs in other domains.

\subsection{Likelihood and Distinguishability: Two sides of the same coin?}

In this work, we suggest a different, fundamental barrier to adopting GANs in domains where MLE is prevalent: 
the adversarial approach of minimizing distinguishability is effectively a roundabout method of maximizing likelihood on observed data, and hence employing MLE directly can be more efficient. This is the case in NLP where, unlike computer vision, a measure of likelihood called \textit{perplexity} has been the prevailing metric for training and evaluating language models for decades. We show how GANs boost likelihood in the spirit of, and inspired by, the related celebrated work of \citet{FHT2000} that showed how boosting can be viewed as an iterative approach for logistic regression.  

Consider a large finite set or countably infinite set $X$ and a family $Q$ of probability distributions over $X$. For language, these might be n-gram models or neural models. Also consider a family $F$ of distinguishers $f: X \rightarrow [0,1]$ that aim to distinguish random examples drawn from a distribution $p$ from those sampled from $q$.  For any such classifier $f$, we call the difference $\alpha(f) = \E_q[f(x)]-\E_p[f(x)]$ the distinguishability \textit{advantage} of $f$ because it quantifies the accuracy of $f$ at the task of identifying ``fake'' examples. A perfect distinguisher would thus have $\alpha(f)=1$, while $f(x)=1/2$ which predicts at random has $\alpha(f)=0$. More formally, imagine picking $y \in \{0,1\}$ uniformly at random and picking a random example $x$ from $q$ if $y=1$ and from $p$ if $y=0$. The (randomized) binary classifier that predicts, for any $x$, $\hat{y}=1$ with probability $f(x)$, has (expected) accuracy: 
\[
\frac{1}{2} \sum_x q(x) f(x) + \frac{1}{2} \sum_x p(x) (1-f(x)) = \frac{1}{2} + \frac{1}{2}\alpha(f).\]
Given a family $F$, we define the \textit{distinguishability} of $q$ from $p$ to be $d(q) = \max_{f\in F} \alpha(f)$. Distinguishability is known to be a lower-bound on \textit{total variation distance} (also called statistical distance),  
a measure of distance between distributions that is difficult to directly estimate for large domains $X$ \citep{sriperumbudur2012empirical}. The ``Bayes-optimal'' distinguisher simply predicts 1 iff $q(x)>p(x)$, and has advantage equal to the total variation distance \citep[see, e.g.,][]{hashimoto2019unifying}. Clearly $d(p)=0$, i.e., $p$ is indistinguishable from itself. Motivated by this observation, numerous \textit{adversarial} approaches to approximating $q$ have been attempted to minimize distinguishability $d(q)$. If $p \in Q$ then $d(q)$ is minimized at $q=p$. 

\paragraph{Example where maximizing likelihood $\neq$ minimizing distinguishability.}
When $p \not\in Q$, minimizing distinguishability among $q \in Q$ may be different than maximizing the likelihood of $q$. For instance, consider modeling the age in years of humans (say the entire population on earth) as a uniform distribution $q_m$ over $x \in \{0, 1, 2, \ldots, m\}$. Now, the $m$ which maximizes likelihood would be the age of the oldest person, which is $m=119$ at the time of this article---any smaller $m$ would assign zero probability to the 119-year-old and thus to the entire population. However, this distribution is very distinguishable from the true distribution---for instance it assigns probability $\sim 17\%$ to being over 100 years, which is extremely unlikely among today's population. A smaller $m < 100$ would yield less distinguishable samples. While it may seem therefore that distinguishability and likelihood are inherently different criteria, as we shall see this is an artificial limitation due to the weakness of family $Q$.

Of course, the (in)equivalence depends on the families $F$ of distinguishers and $Q$ of probability distributions. We give two results showing that maximizing likelihood and minimizing distinguishability are equivalent as long as $F$ and $Q$ are similar in power, even when $p \not\in Q$. First, we consider families $Q$ that are ``log-linear''  over some set $F$ of functions, which include n-gram models and neural networks whose top layer is a softmax, among others. The equivalence in this case is particularly simple and serves to illustrate how MLE can be a simpler way to reach the same optimum. In this case, $Q$ and $F$ are naturally paired.

\paragraph{Maximizing likelihood $=$ minimizing distinguishability for log-linear $Q$.}
In the above age example, the family $Q$ of geometric distributions $q_\theta(n) \propto \exp(-\theta n)$ for $\theta > 0$ is an example of a log-linear family. We show that if $q$ can be distinguished from the population distribution $p$ by a function $f\in F$, then folding $f$ into $q$ yields a new model in $Q$ with greater likelihood. In practice, one only has a sample of the true distribution $p$ (not the entire population) and maximizing log-likelihood is approximation of minimizing KL divergence $\kld{p}{q} = \sum_{x}p(x)\log \frac{p(x)}{q(x)}$. We give formal statements about minimizing KL-divergence as well.

The conclusion of this first observation is that if a GAN were to converge within a log-linear family (and making GANs converge is often not an easy feat in practice), it would converge to the MLE. 

\paragraph{General polynomial-time reduction.}
 Our second result is a polynomial-time reduction from likelihood maximization to next-token distinguishability, without the log-linear requirement. We consider the common case of (unidirectional) sequential models that predict the next token based on the previous tokens, which have several practical advantages including being efficient to compute---the probability of a sequence is simply the product of the conditional probabilities of each subsequent word given the previous words. Many state-of-the-art transformer language models such as GPT-3 take this form. Achieving an efficient reduction is challenging due to the normalization requirement of computing partition functions. In order to achieve a polynomial-time reduction, we consider a notion of next-token distinguishability, where the game is as follows: a prefix of tokens is chosen, based on which the generator generates a token to follow the prefix. Given the actual next token and the generated next token, the distinguisher aims is to identify which is which.  Algorithm 1 leverages a next-token distinguisher to iteratively increase likelihood. In particular, given any target $\epsilon>0$, Theorem \ref{thm:main} shows that Algorithm 1 will terminate and output an efficiently computable model $q$ which is nearly (to within $\eps$) indistinguishable from the truth, and it runs in time polynomial in $1/\epsilon$. 

If $p\in Q$ and one has an optimal distinguisher, one will eventually converge to a model close to $p$, as has been discussed heavily in the literature. However, our results are also meaningful in the more realistic case where one has imperfect distinguishers.

\smallskip
\textbf{Contributions}. The main contributions of this paper are: 
\begin{itemize}
    \item showing that, although in general minimizing distinguishability and maximizing likelihood seem to be different, they are in fact closely related,
    \item introducing a new model of next-token distinguishability that is necessary to make the reduction efficient, and 
    \item offering a new perspective on why GANs are less popular in NLP and other sequential domains as they are for images.
\end{itemize}

\smallskip
\textbf{Organization}. We begin by summarizing related work on GANs, especially for text. We then illustrate how GANs can be overkill for the simple case of n-gram models in Section \ref{sec:overkill}. Section \ref{sec:loglinear} covers log-linear models. Section \ref{sec:main} gives explicit bounds on general reductions between maximizing likelihood and minimizing distinguishability. Section \ref{sec:sequential_distinguishers} shows how the reduction can be efficiently computed in the case of sequential inputs, from which we propose a simple polynomial time algorithm that provably finds a distribution which is nearly-indistinguishable with respect to a given class of discriminator functions. Finally, we discuss the relevance of our work in Section \ref{sec:discussion}.

\section{Related Work}\label{sec:related}

\smallskip
\textbf{Generative models}. 
Several approaches to generative modeling have been investigated, especially in the context of images. In particular, impressive results have been obtained recently with variational autoencoders, GANs, normalizing flows, autoregressive models, diffusion processes, and score/energy based models \cite{VAEs, G2014, NormalizingFlows, PixelCNN2016, Diffusion2020, song2021scorebased, score2021, NeuralODEs}. Generally, training approaches are either adversarial; or rely on MLE, contrastive divergence estimation, or score matching \cite{song2021maximum}. Some connections have begun to emerge between these models, and alternate training procedures have been advocated  \cite{CZSLPCB2020, song2021scorebased, yair2021contrastive}.  \\

\smallskip 
\textbf{GANs for text}.
Since their introduction \citep{G2014}, there has been interest in adapting GANs to text data. The driving motivation was that\;---up until very recently---\;samples generated by traditional (likelihood-based) models had been easy to distinguish from human-generated text, and the success of image GANs at generating realistic-looking samples suggested a possible avenue to improve the quality of their natural language counterparts.

The first and most significant challenge in adapting GANs to text arises from the very nature of this data. Goodfellow \cite{goodfellow2017nips} points out that GANs require the generator to be differentiable, which poses a challenge for discrete text representations such as one-hot word
or character representations. Two of the most popular approaches to circumvent this obstacle are policy gradient techniques (e.g., REINFORCE \citep{williams1992simple}) ---which when applied to language modeling nevertheless often require maximum likelihood pre-training \citep{che2017maximum, yu2017seqgan})--- and the Gumbel-Softmax approximation \citep{kusner2016gans}. The few adversarial methods that do not require pre-training (e.g., \citep{press2017language, rajeswar2017adversarial}) have failed to show significant promise in all but a few artificial tasks.

This nascent but active line of work seemed to suggest for a period of time that GANs might provide a breakthrough in text generation. This promise did not fully materialize, and instead the most recent breakthrough came from models building very large transformer-based architectures like GPT \citep{radford2018improving, radford2019language, brown2020language} or PaLM \citep{PaLM2022} --- which are trained with traditional cross-entropy (MLE) objectives.

Yet the question of how GAN-based methods for text compare with likelihood-based ones still garners significant interest, and while various works have provided an empirical comparison between them ---with most of these suggesting the advantage of MLE-based ones \citep{caccia2018language}--- theoretical explanations have been less explored. 

\smallskip 
\textbf{Relating objectives via divergences}.
The connection between maximum likelihood estimation, distinguishability and divergences between probability distributions has been explored before. For example, it is well known that maximizing likelihood is equivalent to minimizing the KL divergence between certain families of fitted and reference distributions, though this is not the only divergence for which such a connection exists \citep{rigollet2018entropic}. On the other hand, from the moment GANs were introduced, \citet{G2014} noted that\;---assuming a perfect discriminator---\;the adversarial objective corresponds to minimizing a Jensen-Shannon divergence. Furthermore, the minimal discrimination error is also directly related to the total variation distance (see, e.g., \citet{hashimoto2019unifying}). On the other hand, for exponential families the gradient of the KL divergence is known to be related to the discrepancy between distributions \citep{TH2015}. While conceptually similar to this line of work, here instead we give an \textit{explicit} reduction that shows how distinguishability and (log) likelihood are in direct correspondence.

Pinsker's inequality is a well-known result linking KL divergence and total variation distance (TVD):  $\text{TVD} \leq \sqrt{\text{KL}/2}$. While related, this inequality is not directly relevant to the context of this work. First, while total variation provides an upper bound to distinguishability, it is not computable in general, so it is rarely used as a training objective for generative models. On the other hand, being one-sided,\footnote{Reverse Pinsker's inequalities exist only for particular cases, but they too are very loose in general \citep{sason2015reverse}.} it does not imply that reducing TVD reduces KL divergence. Furthermore, Pinsker's is in general a very loose inequality, particularly for the direction of KLD that is equivalent to MLE (i.e., $\kld{p}{q_{\theta}}$), since if $p(x) > 0 \approx q_{\theta}(x)$ even for a single $x$ leads to unbounded KL divergence. In contrast, in this work we provide a \textit{direct reduction} directly linking the two criteria of interest: distinguishability and maximum likelihood.


\smallskip 
\textbf{Log-linear language models}. In this work we focus our analysis on log-linear models \citep{LMP2001, MFP2000}, which are widely used in natural language processing (often known in that community as Maximum Entropy --MaxEnt-- models) for various tasks. In particular, these models have been a cornerstone of both neural \citep{BDVJ2003, mikolov2013efficient} and non-neural \citep{rosenfeld1994adaptive, khudanpur2000maximum} language modeling. 

\smallskip 
\textbf{Boosting}. The reduction shown here bears resemblance to boosting. It is well-known that boosting can be analyzed through the lens of maximum likelihood  (\citet{FHT2000}), while \citet{LL2002} formalized the equivalence of AdaBoost and maximum likelihood training for exponential models. More recently, boosting has been applied to generative adversarial models \citep{tolstikhin2017adagan, grover2018boosted}, albeit with a different approach and objective than the connection drawn in this work.




\section{Illustration: GANs for n-gram language models}\label{sec:overkill}
To illustrate our main point, consider first the simplest model of language: a unigram model where the probability of each word is independent, and the end of sentence token $\EOS$ has a given probability as well. If $\theta_w$ represents the log-probability of word $w$, then the log-probability of sentence $w_1\ldots w_t$ is given by:
$$\log q(w_1 w_2 \ldots w_t)= \theta_{w_1} + \theta_{w_2} + \ldots + \theta_{w_t} + \theta_{\EOS}.$$ 
The MLE parameters $\theta^*$ can be computed in linear time by simply counting word frequencies.  

A more roundabout approach to fitting a unigram language model would be to start with any initial unigram model $q$, generate random samples from $q$ and compare them to those from $p$. One could then distinguish the two by finding a word that appears significantly more often in one than in the other. For example, if one generates text from the model $q$ and finds that the word ``the'' occurs much more often in text generated from $p$, one would then update the parameters by increasing $\theta_{\mathrm{the}}$ (and decreasing $\theta_{w'}$ for all other words $w'$ to keep $q$ a probability distribution). As we shall see later, if this more involved procedure converged, it would necessarily converge to the same maximum-likelihood estimator $\theta^*$.

A similar argument applies to any $n$-gram model in which the probability of each subsequent word is determined only by the previous $n-1$ words. This is also optimized by frequency counts (a variety of ``smoothing'' techniques, e.g., adding 1 to counts, also known as Laplace Smoothing \citep{good1953population, kneser1995improved} are often used as a form of regularization on top of these counts). Distinguishers could similarly be used to find a model $q$ that is indistinguishable from $p$ according to $n$-gram frequencies, but again this would simply converge to the MLE parameters.

\section{Equivalence for log-linear models}\label{sec:loglinear} 
In this section, we show that there is one optimal log-linear model that both minimizes distinguishability and maximizes log-likelihood. Consider a log-linear model with features $f:X \rightarrow [0,1]^d$, i.e., $d$ bounded features $f_i:X\rightarrow [0,1]$. The model predicts
\begin{equation}
q_\theta(x)=\frac{\exp\big(\langle f(x), \theta \rangle\bigr)}{Z_{\thetavec}}~,
\end{equation}
where $\langle \cdot, \cdot\rangle$ denotes inner product, $\thetavec\in \mathbb{R}^d$ is a parameter vector and $Z_{\thetavec}=\sum_x \exp\langle \thetavec, f(x)\rangle$ is a normalizing constant called the partition function. 

In the unigram example, the features $f_i$ would be word counts normalized by dividing by the maximum allowed sentence length (to ensure $f_i(x)\leq 1$). In a neural network the features $f_i$ would correspond to the top layer and $q$ computes a softmax. Multiple strategies have been studied for computing or estimating the partition function $Z_{\thetavec}$ \citep[see, e.g.,][]{desjardins2011tracking}. 

As discussed earlier, these feature functions can also be thought of as classifiers that distinguish  examples drawn from $p$ from those drawn from $q_{\theta}$ and the advantage of $f_i$ is $\alpha(f_i)=\sum_x f_i(x)(q_{\theta}(x)-p(x))$. The advantage vector is $\alpha(f)=\langle \alpha(f_i) \rangle_{i=1}^d$. Note that a negative advantage can be used for distinguishing by using the reverse classifier $1-f_i$ as a distinguisher, which has opposite advantage $\alpha(1-f_i)=-\alpha(f_i)$. 

\begin{observation}\label{obs:loglinear}
The gradient of $\kld{p}{q_\thetavec}$ with respect to $\thetavec$ is the advantage vector $\alpha(f)$, i.e., for all $i=1,2,\ldots, d$:
$$\frac{\partial \kld{p}{q_\thetavec}}{\partial \theta_i} = \sum_x f_i(x)(q_\thetavec(x)-p(x)) = \alpha(f_i).$$
\end{observation}
The above straightforward calculation is well-known as is the fact that $\kld{p}{q_\thetavec}$ is convex in $\thetavec$. However, we interpret this fact in the context of GANs: searching for $\thetavec$ which gives a zero-gradient for KL divergence is equivalent to finding $\thetavec$ which is indistinguishable with respect to each $f_i$. While a number of GANs have be designed in various architectures that solve the seemingly more complex problem of $\min_\thetavec d(q_\thetavec)$, it can generally be more efficient to maximize likelihood, which (approximately) minimizes the KL divergence. 

\section{Distinguishability is equivalent to increasing likelihood for general $F$, $Q$}\label{sec:main}
In this section, we show how reducing log-loss is equivalent to distinguishing real and generated examples. This is the basis behind a single step of our main algorithm (the reduction in this section is efficient for a single step, but the increase in runtime would lead to a general exponential-time algorithm). The bounds here are in terms of log-loss, as measured on a specific sample, rather than the abstract KL divergence quantity of the previous section, which cannot be computed exactly using a finite sample. In particular, we show how, if one can distinguish a given distribution from the sample, then one can decrease that distribution's log-loss, and vice versa.

For the remainder of this section, we drop $\theta$ from the variable denoting the fitted distribution $q_{\theta}$ to avoid cluttering the notation. Fix a sample $\mathcal{S} = \langle x_1, \ldots, x_n \rangle \in X^n$ 
of $n$ training examples drawn from $p$, and define the log-loss to be: 
\[ \hat{L}(q; \mathcal{S}) = -\frac{1}{n}\sum_{i=1}^n \log q(x_i)= -\hE_{\mathcal{S}}[\log q(x)],\]
where we use hat on $L$ to denote that the loss is estimated on a (finite) training set $\cS$. Likewise,  $\hE_{\mathcal{S}} [g(x)]$ denotes the empirical expectation $\frac{1}{n}\sum_{i=1}^n g(x_i)$. Note that the expected log-loss over training sets is known as the cross-entropy 
\[ H(p,q) = \E_{\mathcal{S} \sim p^n}[\hat{L}(q; \mathcal{S})],\]
and hence the expected difference in log-loss between two candidate distributions $q$ and $q'$ is equal to the difference \[  \E_{\mathcal{S}\sim p^n}[\hat{L}(q; \mathcal{S}) - \hat{L}(q'; \mathcal{S})] = \kld{p}{q} - \kld{p}{q'}~,\]
so minimizing log-loss approximately minimizes the KL divergence. Also, we define the \textit{training advantage} 
of distinguisher $f: X \rightarrow [0,1]$ to be:
\begin{equation}\label{eq:advantage}
  \hat\alpha(f) = \E_q[f(x)]-\hE_{\mathcal{S}}[f(x)]  
\end{equation}
which is independent of $p$, depending on the sample alone and can thus be estimated to arbitrary accuracy using samples generated from $q$. The lemmas below show how one can use a distinguisher to reduce log-loss on the same training sample, and how to use a distribution with a lower log-loss to distinguish the two distributions. 

\begin{lemma}\label{lem:decreaseLogloss}
Let $a\geq 0$ and suppose $f:X \rightarrow [0,1]$ has training advantage $\hat\alpha(f) \geq a.$ Then, the probability distribution $q'(x) = q(x)e^{-a f(x)}/Z_{q'}$ where $Z_{q'}=\sum_x q(x)e^{-a f(x)}$, has lower log-loss:
$$\hat{L}(q'; \mathcal{S}) \leq \hat{L}(q; \mathcal{S}) - a^2/2.$$
\end{lemma}
Before we give the proof, we note that if $f$ is computed by a neural network and $q$ is computed as a neural network with a softmax at the top layer, i.e., $q(x) = e^{\langle v, g(x)\rangle}/\sum_x e^{\langle v, g(x)\rangle}$ where $g:X \rightarrow \reals^d$ is some neural network, then $q'$ is naturally represented as the combined neural network with softmax $q'(x) \propto e^{\langle (v, -a), (g(x), f(x))\rangle}$ in $d+1$ dimensions.
\begin{proof}[Proof (Lemma \ref{lem:decreaseLogloss})]
\begin{align}
    \hat{L}(q; \mathcal{S}) - \hat{L}(q'; \mathcal{S}) & = \hE_{ \mathcal{S}}[\log q'(x) -\log q(x)]\nonumber\\
    &= \hE_{ \mathcal{S}}\left[-a f(x) - \log Z_{q'} \right]\nonumber\\
    &= a (\E_q\left[ f(x) \right]-\hE_{ \mathcal{S}}\left[ f(x) \right])
    - a \E_q\left[ f(x) \right] - \log Z_{q'}\nonumber\\
    &=  a \hat\alpha(f)   -a \E_q\left[ f(x) \right] - \log Z_{q'}  \label{eq:foo}
\end{align}
Since $\hat\alpha(f)\geq a$ by assumption, it remains to show $a \E_q\left[ f(x) \right] + \log Z_{q'}\leq a^2/2$. Using the bound $\log r\leq r-1$ for any $r>0$, 
we get that,
\begin{align*}
    a \E_q\left[ f(x) \right] + \log Z_{q'}
     &\leq  a \E_q\left[ f(x) \right] + Z_{q'}-1\\ 
     &= a \E_q\left[ f(x) \right] + \E_q[e^{-a f(x)}] - 1\\
     &= \E_q\left[ a f(x) + e^{-a f(x)} - 1\right]\\
     &\leq \E_q[(a f(x))^2/2],
\end{align*}
where we have used the fact that $Z_{q'}=\E_q[e^{-a f(x)}]$ and, to get to the last line we use $e^{-r}+r-1\leq r^2/2$ for $r \geq 0$ by Taylor expansion. Since $f(x)\in [0,1]$, the last quantity is at most $a^2/2$, which together with (\ref{eq:foo}), gives $\hat{L}(q; \mathcal{S})-\hat{L}(q'; \mathcal{S}) \geq a^2/2$~.
\end{proof}
This means that if we can distinguish $\mathcal{S}$ from $q$, then we can simply reduce log-loss by down-weighting suspicious samples that satisfy the distinguisher $f(x)$. The difference between this statement and Observation \ref{obs:loglinear} is analogous to the difference between boosting and logistic regression \citep{FHT2000}. In logistic regression, one typically fixes the set of features in advance, whereas in boosting this is not possible if there are infinitely many possible classifiers. 


Conversely, we next show that if $q'$ has a lower log-loss than $q$ on the training samples, then we can distinguish $q$ from these samples.
\begin{lemma}\label{lem2}
For any constant $C > 1$ and distributions $q, q'$ such that $\frac{1}{C} q(x) \leq q'(x) \leq C q(x)$ for all $x \in X$, the distinguisher $f: X \rightarrow [0,1]$ defined by, 
\[f(x) = \frac{1}{2\log C}\log \frac{C q(x)}{q'(x)},\] has a training advantage of,
\[\hat\alpha(f) \geq \frac{\hat{L}(q) - \hat{L}(q')}{2\log C}.\]
\end{lemma}
\begin{proof}[Proof (Lemma \ref{lem2})]
Let $g(x) = \log q(x) - \log q'(x)$. By Jensen's inequality,
\begin{align*}
\E_q[g(x)] - \hE_{\mathcal{S}}[g(x)] &= -\E_q\left[\log \frac{q'(x)}{q(x)}\right] - \hE_{\mathcal{S}}[g(x)]\\ 
&\geq -\log \E_q\left[\frac{q'(x)}{q(x)}\right] - \hE_{\mathcal{S}}[g(x)]\\
&= -\log(1)- \hE_{\mathcal{S}}[g(x)]\\
&= \E_{\mathcal{S}}[\log q'(x)] - \hE_{\mathcal{S}}[\log q(x)]\\
&= \hat{L}(q; S) - \hat{L}(q'; S)
\end{align*}
Since $f(x) = \frac{1}{2\log C}(g(x) + \log C)$, the training advantage of $f$ is that of $g$ scaled by a factor of $\frac{1}{2\log C}$. Finally, it is straightforward to verify that $f(x) \in [0,1]$ by our assumptions on the ratio between $q$ and $q'$. 
\end{proof}

Importantly, due to the logarithmic dependence on $C$, the above lemma is meaningful even if $q$ and $q'$ are exponentially far apart so long as they have the same support.

Lemma \ref{lem:decreaseLogloss} implies a reduction between the problem of distinguishing with nontrivial advantage to non-trivially reducing log-loss for log-linear families.  Note that iteratively applying the reduction requires repeated computation of the normalization terms over $X$, and computing such partition functions is an area of active research---where it is known how to do it efficiently for some classes and not for others. The next section gives an efficient reduction for (unidirectional) sequential models.

\section{Efficient Reduction for Sequential Models}\label{sec:sequential_distinguishers}

This section gives an efficient reduction from distinguishing to MLE for sequential models. This requires showing how one can efficiently compute the normalization terms (partition function) on a token-by-token basis for black-box sequential (e.g., language or auto-regressive) models. The key insight for efficiency is that, rather than distinguishing entire sequences from $p$ and $q$, one distinguishes the conditional next-token predictions. In particular, rather than generating entire sequences from $q$, one can generate next-token predictions on all sequence prefixes in the training data. 

Clearly, evaluating a neural network over all sequences is infeasible. However, in many applications such as NLP, the inputs are sequential $x=(x_1,\dots,x_{\ell})$, where every token $x_i$ is taken from a large discrete vocabulary. In such cases, the combinatorial nature of the data makes density estimation intractable unless the likelihood computation is broken into small sequential steps by representing the overall probability as the $\Pr(x) = \prod_j \Pr(x_j | x_1, x_2, \ldots x_{j-1})$ . 

In this section we show how a natural extension of the framework described above allows us to achieve an efficient reduction for this common type of sequential model. To do so, we define a simple generalization of the training advantage criterion \eqref{eq:advantage}, which now relies on a \textit{step-wise distinguisher} $g$ operating on variable-length sequences. Formally, we consider a \textit{language} of $N$-length sequences\footnote{Padding can be used to handle sequences of variable length.} of tokens taken from a vocabulary $\mathcal{V}$, and distinguisher functions $g:\bigcup_{j=1}^N \mathcal{V}^j \rightarrow [0,1]$, i.e., functions which can take subsequences of any size as input. Given a sample $\cS$ of sequences, we say that $g$ has \textit{generalized training advantage} given by
\begin{equation}\label{eq:advantage_partial_distinguisher}
    \hat{\beta}(g) = \hat{\beta}(g, \cS, q) = \frac{1}{N}\sum_{j=1}^{N} \hE_{x\sim \mathcal{S}}  \bigl[ \E_{w \sim q(\cdot|x_0,\dots,x_{j-1})} g(x_0,\dots,x_{j-1},w) - g(x_0, \dots, x_{j}) \bigr]
\end{equation}
where, by convention, $x_0=\emptyset$, so that $q(x_0, w)=q(w)$. This criterion can be interpreted as follows. For every length $j \in \{1, 2, \ldots, N\}$, $g$ is tasked with distinguishing a subsequence consisting of the first $j$ tokens in a \textit{true} sequence sampled from $\cS$ from another $j$-length sequence in which the last element is replaced by a randomly selected token from the alternate distribution $q$.

\begin{lemma}\label{lem:decreaseLogloss_partial}
Let $b\geq 0$ and suppose $g:\bigcup_{j=1}^N \mathcal{V}^j \rightarrow [0,1]$ has generalized training advantage $\hat\beta(g) \geq b$. We define a distribution $q'$ through its conditional probabilities as:
\begin{equation*}
    q'(x_{j}\mid x_1, \dots, x_{j-1}) = q(x_{j} \mid x_1,\dots, x_{j-1}) e^{-bg(x_1,\dots, x_{j})}/Z_{q'}(x_1,\dots,x_{j-1})
\end{equation*}
where now $Z_{q'}(x_1, \dots, x_{j-1}) =\sum_{\tilde{x}_{j}} q(\tilde{x}_{j} \mid x_1, \dots, x_{j-1})e^{-b g(x_1, \dots, x_{j-1}, \tilde{x}_{j})}$. Then $q'$ incurs lower log-loss than $q$:
$$\hat{L}(q'; \mathcal{S}) \leq \hat{L}(q; \mathcal{S}) - Nb^2/2.$$
\end{lemma}
The proof is deferred to Appendix~\ref{sec:proof_lemma_seq}.

Next, we use Lemma~\ref{lem:decreaseLogloss_partial} repeatedly to derive a simple algorithm that, given access to non-trivial weak distinguishers, returns a distribution that is nearly indistinguishable (by that class) from the true distribution $p$. Formally, let $\mathcal{G} = \{ g \medspace | \medspace g:\bigcup_{j=1}^N \mathcal{V}^j \rightarrow [0,1]\}$ be a class of distinguishers.  We assume access to an oracle $\mathrm{O}_d: \mathcal{Q} \mapsto \mathcal{G}$ which for any $q \in \mathcal{Q}$ returns a distinguisher $g$. In practice, such as in typical GAN training setting, one could think of this oracle as being approximated by the subroutine that trains the discriminator. We say that $q$ is $\epsilon$-indistinguishable by oracle $\mathrm{O}_d$ if its output $g$ has advantage $\hat{\beta}(g, \cS, q)\leq \eps$. 
We do not need to assume that $\mathrm{O}_d$ is optimal in any sense.


\begin{algorithm2e}
\caption{Boosted weak distinguishers.}
\DontPrintSemicolon
\SetAlgoLined
\SetKwInOut{Input}{input}
\KwIn{Initial model $q_0$, corpus $\cS,$ distinguisher oracle $O_d$, advantage threshold $\epsilon$.}
 $t\gets 0$\;
 \While{True}{
 $g_t \gets \mathrm{O}_d(q_t)$\;
 $b_t \gets \hat{\beta}(g_t, \cS, q_t)$\;
 if $b_t < \epsilon$, \KwOut{$q_t$}
 Compute $q_{t+1}(x) \triangleq q_t(x) e^{-b_tg_t(x)}/\sum_{x \in \cS}q_t(x) e^{-b_tg_t(x)}$ on entire corpus $\mathcal{S}$\;
 $t \gets t+1$\; 
 }
\end{algorithm2e}

\begin{theorem}\label{thm:main}
Let $q_0$ be a language model and let $\eps>0$. Algorithm 1 returns a distribution $q^*$ which is $\epsilon$-indistinguishable from $\cS$ by oracle $\mathrm{O}_d$. It runs in $O\bigl(\frac{1}{\epsilon^2}L_0(\frac{T_d}{N} + nT_g(m+n))\bigr)$ time, where $L_0 = \hat{L}(q_0; \mathcal{S})$ is the log-loss of $q_0$, $T_d$ is the runtime of oracle $\mathrm{O}_d$, $T_g$ is the complexity of evaluating any distinguisher $g$ on a single input, $n=|\mathcal{V}|$ is the vocabulary size, $N$ is the sequence length and $m=|\cS|$ is the number of training sequences. 
\end{theorem}
\begin{proof}
    The fact that Algorithm 1 terminates with a distribution $q$ which is $\epsilon$-indistinguishable by $\mathrm{O}_d$ is immediate from the stopping criterion. 
    
    Now, for the runtime analysis, note that ---by construction--- the iterates $g_t$, $t\in\{0,\dots,T-1\}$ have training advantage $\hat\beta(g_t, \cS, q_t) \geq \epsilon$. Thus, by Lemma~\ref{lem:decreaseLogloss_partial}, the algorithm makes at least $\frac{N\epsilon^2}{2}$ improvement in each iteration. Therefore, the total number of iterations $T$ is at most $\frac{2L_0}{N\epsilon^2}$, where $L_0 := \hat{L}(q_0; \mathcal{S})$ is the log-loss of the initial model. Each iteration of Algorithm 1 requires calling $\mathrm{O}_d$ oracle once, evaluating $\hat\beta( \cdot)$ at an $\mathcal{O}(NmnT_g)$ complexity, and updating each of the $n$ next-token probabilities of $q$ for each sequence length $1,\dots,N$. Each of these updates involves evaluating $g$ plus an $\mathcal{O}(n)$ partition normalization. Putting these together, we conclude that each iteration has $O(T_d + NnT_g(m + n))$ complexity.  
    
    Combining the the two arguments above, we conclude that Algorithm 1 has a total runtime of $O\bigl(\frac{1}{\epsilon^2}L_0(\frac{T_d}{N} + nT_g(m+n))\bigr)$.
\end{proof}

Thus, Algorithm 1 combined with Theorem~\ref{thm:main} and Lemma~\ref{lem:decreaseLogloss_partial} yields a polynomial-time reduction from distinguishing distributions to maximum likelihood estimation for sequential models. 


\section{Discussion and conclusions}\label{sec:discussion}
In this work, we have argued that minimizing log-loss (i.e., KL-divergence) and minimizing statistical distinguishability are tightly related goals. Specifically, if the families of distinguishers and probability distributions are of similar power, then one can use a distinguisher to reduce log-loss. This means in applications where it is natural to fit models by minimizing log-loss,  it is indeed likely to be a more direct and efficient means of fitting a model. This is the case for n-gram language models (and other sequential tasks), for which perplexity (a measure of likelihood) 
is easy to compute,  naturally meaningful, and allows for efficient sampling. Thus, for a long time, minimizing log-loss has been the objective with which most state-of-the-art models are trained. For such models, Lemma \ref{lem:decreaseLogloss} implies that if one can distinguish the model from samples by a neural network then one can construct a larger neural network with lower log-loss. Hence, one may prefer to simply train a larger model in the first place. 


 %

\section*{Broader Impact}\vspace{-0.2cm}
The contribution of this work is conceptual and theoretical, and as such, any nuanced discussion of the potential harms or benefits of its impact are irremediably tied to the applications where it might be put to use. We first make a few general observations regarding its immediate impact, and then discuss in a more informal manner downstream ramifications these might have in applications. 

This work revolves around comparing to training paradigms: maximum likelihood and adversarial learning. We believe that the maxim of `choosing the right tool for the job' applies in this context too, and can have important downstream consequences. For example, the amount of resources consumed by training large generative models has been growing substantially over the past few years \citep{amodei2018ai}. This is particularly true for Natural Language Processing, where state-of-the-art models are increasingly large and trained on increasingly larger datasets, leading to striking computationally and environmental costs \citep{strubell2019energy}. The key takeaway offered by this work, namely that training certain generative models like language models through adversarial methods is less efficient than doing so via likelihood maximization approaches, could potentially lead to significant saving of these resources by steering practitioners away from adversarial approaches. On the other hand, it is not our intention for this work to lead to the opposite\;---but equally undesirable---\;effect of dissuading practitioners from choosing adversarial training approaches whenever those are a sensible choice. 
\printbibliography

@incollection{LL2002,
title = {Boosting and Maximum Likelihood for Exponential Models},
author = {Guy Lebanon and John D. Lafferty},
booktitle = {Advances in Neural Information Processing Systems (NIPS)},
editor = {T. G. Dietterich and S. Becker and Z. Ghahramani},
pages = {447--454},
year = {2002}
}

@ARTICLE{sriperumbudur2012empirical,
  title     = "On the empirical estimation of integral probability metrics",
  author    = "Sriperumbudur, Bharath K and Fukumizu, Kenji and Gretton, Arthur
               and Schölkopf, Bernhard and Lanckriet, Gert R G",
  journal   = "Electron. J. Stat.",
  publisher = "Institute of Mathematical Statistics and Bernoulli Society",
  volume    =  6,
  number    = "none",
  pages     = "1550--1599",
  year      =  2012,
  doi       = "10.1214/12-EJS722"
}

@inproceedings{VAEs,
  author    = {Diederik P. Kingma and
               Max Welling},
  editor    = {Yoshua Bengio and
               Yann LeCun},
  title     = {Auto-Encoding Variational Bayes},
  booktitle = {2nd International Conference on Learning Representations, {ICLR} 2014,
               Banff, AB, Canada, April 14-16, 2014, Conference Track Proceedings},
  year      = {2014},
  url       = {http://arxiv.org/abs/1312.6114},
  bibsource = {dblp computer science bibliography, https://dblp.org}
}

@article{NormalizingFlows,
  author    = {George Papamakarios and
               Eric T. Nalisnick and
               Danilo Jimenez Rezende and
               Shakir Mohamed and
               Balaji Lakshminarayanan},
  title     = {Normalizing Flows for Probabilistic Modeling and Inference},
  journal   = {J. Mach. Learn. Res.},
  volume    = {22},
  pages     = {57:1--57:64},
  year      = {2021},
  url       = {http://jmlr.org/papers/v22/19-1028.html},
  bibsource = {dblp computer science bibliography, https://dblp.org}
}

@inproceedings{PixelCNN2016,
 author = {van den Oord, Aaron and Kalchbrenner, Nal and Espeholt, Lasse and kavukcuoglu, koray and Vinyals, Oriol and Graves, Alex},
 booktitle = {Advances in Neural Information Processing Systems},
 editor = {D. Lee and M. Sugiyama and U. Luxburg and I. Guyon and R. Garnett},
 pages = {},
 publisher = {Curran Associates, Inc.},
 title = {Conditional Image Generation with PixelCNN Decoders},
 url = {https://proceedings.neurips.cc/paper/2016/file/b1301141feffabac455e1f90a7de2054-Paper.pdf},
 volume = {29},
 year = {2016}
}

@inproceedings{CZSLPCB2020,
 author = {Che, Tong and ZHANG, Ruixiang and Sohl-Dickstein, Jascha and Larochelle, Hugo and Paull, Liam and Cao, Yuan and Bengio, Yoshua},
 booktitle = {Advances in Neural Information Processing Systems},
 editor = {H. Larochelle and M. Ranzato and R. Hadsell and M.F. Balcan and H. Lin},
 pages = {12275--12287},
 publisher = {Curran Associates, Inc.},
 title = {Your GAN is Secretly an Energy-based Model and You Should Use Discriminator Driven Latent Sampling},
 url = {https://proceedings.neurips.cc/paper/2020/file/90525e70b7842930586545c6f1c9310c-Paper.pdf},
 volume = {33},
 year = {2020}
}

@inproceedings{NeuralODEs,
 author = {Chen, Ricky T. Q. and Rubanova, Yulia and Bettencourt, Jesse and Duvenaud, David K},
 booktitle = {Advances in Neural Information Processing Systems},
 editor = {S. Bengio and H. Wallach and H. Larochelle and K. Grauman and N. Cesa-Bianchi and R. Garnett},
 pages = {},
 publisher = {Curran Associates, Inc.},
 title = {Neural Ordinary Differential Equations},
 url = {https://proceedings.neurips.cc/paper/2018/file/69386f6bb1dfed68692a24c8686939b9-Paper.pdf},
 volume = {31},
 year = {2018}
}

@inproceedings{
yair2021contrastive,
title={Contrastive Divergence Learning is a Time Reversal Adversarial Game},
author={Omer Yair and Tomer Michaeli},
booktitle={International Conference on Learning Representations},
year={2021},
url={https://openreview.net/forum?id=MLSvqIHRidA}
}

@inproceedings{score2021,
 author = {Vahdat, Arash and Kreis, Karsten and Kautz, Jan},
 booktitle = {Advances in Neural Information Processing Systems},
 editor = {M. Ranzato and A. Beygelzimer and Y. Dauphin and P.S. Liang and J. Wortman Vaughan},
 pages = {11287--11302},
 publisher = {Curran Associates, Inc.},
 title = {Score-based Generative Modeling in Latent Space},
 url = {https://proceedings.neurips.cc/paper/2021/file/5dca4c6b9e244d24a30b4c45601d9720-Paper.pdf},
 volume = {34},
 year = {2021}
}

@inproceedings{
song2021maximum,
title={Maximum Likelihood Training of Score-Based Diffusion Models},
author={Yang Song and Conor Durkan and Iain Murray and Stefano Ermon},
booktitle={Advances in Neural Information Processing Systems},
editor={A. Beygelzimer and Y. Dauphin and P. Liang and J. Wortman Vaughan},
year={2021},
url={https://openreview.net/forum?id=AklttWFnxS9}
}

@inproceedings{
song2021scorebased,
title={Score-Based Generative Modeling through Stochastic Differential Equations},
author={Yang Song and Jascha Sohl-Dickstein and Diederik P Kingma and Abhishek Kumar and Stefano Ermon and Ben Poole},
booktitle={International Conference on Learning Representations},
year={2021},
url={https://openreview.net/forum?id=PxTIG12RRHS}
}

@inproceedings{Diffusion2020,
 author = {Ho, Jonathan and Jain, Ajay and Abbeel, Pieter},
 booktitle = {Advances in Neural Information Processing Systems},
 editor = {H. Larochelle and M. Ranzato and R. Hadsell and M.F. Balcan and H. Lin},
 pages = {6840--6851},
 publisher = {Curran Associates, Inc.},
 title = {Denoising Diffusion Probabilistic Models},
 url = {https://proceedings.neurips.cc/paper/2020/file/4c5bcfec8584af0d967f1ab10179ca4b-Paper.pdf},
 volume = {33},
 year = {2020}
}

@inproceedings{TH2015,
 author = {Theis, Lucas and Hoffman, Matthew D.},
 title = {A Trust-region Method for Stochastic Variational Inference with Applications to Streaming Data},
 booktitle = {International Conference on Machine Learning (ICML)},
 year = {2015},
 pages = {2503--2511},
 numpages = {9}
}

@incollection{G2014,
title = {Generative Adversarial Nets},
author = {Goodfellow, Ian and Pouget-Abadie, Jean and Mirza, Mehdi and Xu, Bing and Warde-Farley, David and Ozair, Sherjil and Courville, Aaron and Bengio, Yoshua},
booktitle = {Advances in Neural Information Processing Systems (NIPS)},
pages = {2672--2680},
year = {2014}
}

@InProceedings{ZGMO2019,
  title = 	 {Self-Attention Generative Adversarial Networks},
  author = 	 {Zhang, Han and Goodfellow, Ian and Metaxas, Dimitris and Odena, Augustus},
  booktitle = 	 {International Conference on Machine Learning (ICML)},
  pages = 	 {7354--7363},
  year = 	 {2019}
}

@article{BDVJ2003,
 author = {Bengio, Yoshua and Ducharme, R{\'e}jean and Vincent, Pascal and Janvin, Christian},
 title = {A Neural Probabilistic Language Model},
 journal = {Journal of Machine Learning Research (JMLR)},
 issue_date = {3/1/2003},
 volume = {3},
 year = {2003},
 pages = {1137--1155},
 numpages = {19}
}

@article{FHT2000,
  author = {Friedman, J. and Hastie, T. and Tibshirani, R.},
  journal = {The Annals of Statistics},
  number = 2,
  title = {{Additive Logistic Regression: a Statistical View of Boosting}},
  volume = 38,
  year = 2000
}

@inproceedings{LMP2001,
 author = {Lafferty, John D. and McCallum, Andrew and Pereira, Fernando C. N.},
 title = {Conditional Random Fields: Probabilistic Models for Segmenting and Labeling Sequence Data},
 booktitle = {International Conference on Machine Learning (ICML)},
  year = {2001},
 pages = {282--289},
}

@inproceedings{MFP2000,
 author = {McCallum, Andrew and Freitag, Dayne and Pereira, Fernando C. N.},
 title = {Maximum Entropy Markov Models for Information Extraction and Segmentation},
 booktitle = {International Conference on Machine Learning (ICML)},
 year = {2000},
 pages = {591--598}
}

@article{good1953population,
  title={The population frequencies of species and the estimation of population parameters},
  author={Good, Irving J},
  journal={Biometrika},
  volume={40},
  number={3-4},
  pages={237--264},
  year={1953},
  publisher={Oxford University Press}
}

@inproceedings{kneser1995improved,
  title={Improved backing-off for m-gram language modeling},
  author={Kneser, Reinhard and Ney, Hermann},
  booktitle={1995 International Conference on Acoustics, Speech, and Signal Processing},
  volume={1},
  pages={181--184},
  year={1995},
  organization={IEEE}
}

@article{dautume2019training,
  title={Training language GANs from Scratch},
  author={d'Autume, Cyprien de Masson and Rosca, Mihaela and Rae, Jack and Mohamed, Shakir},
  journal={arXiv preprint arXiv:1905.09922},
  year={2019}
}

@inproceedings{tevet2019evaluating,
    title = "Evaluating Text {GAN}s as Language Models",
    author = "Tevet, Guy  and
      Habib, Gavriel  and
      Shwartz, Vered  and
      Berant, Jonathan",
    booktitle = "Proceedings of the 2019 Conference of the North {A}merican Chapter of the Association for Computational Linguistics: Human Language Technologies, Volume 1 (Long and Short Papers)",
    month = jun,
    year = "2019",
    address = "Minneapolis, Minnesota",
    publisher = "Association for Computational Linguistics",
    doi = "10.18653/v1/N19-1233",
    pages = "2241--2247"
}

@misc{PaLM2022,
  doi = {10.48550/ARXIV.2204.02311},
  url = {https://arxiv.org/abs/2204.02311},
  author = {Chowdhery, Aakanksha and Narang, Sharan and Devlin, Jacob and Bosma, Maarten and Mishra, Gaurav and Roberts, Adam and Barham, Paul and Chung, Hyung Won and Sutton, Charles and Gehrmann, Sebastian and Schuh, Parker and Shi, Kensen and Tsvyashchenko, Sasha and Maynez, Joshua and Rao, Abhishek and Barnes, Parker and Tay, Yi and Shazeer, Noam and Prabhakaran, Vinodkumar and Reif, Emily and Du, Nan and Hutchinson, Ben and Pope, Reiner and Bradbury, James and Austin, Jacob and Isard, Michael and Gur-Ari, Guy and Yin, Pengcheng and Duke, Toju and Levskaya, Anselm and Ghemawat, Sanjay and Dev, Sunipa and Michalewski, Henryk and Garcia, Xavier and Misra, Vedant and Robinson, Kevin and Fedus, Liam and Zhou, Denny and Ippolito, Daphne and Luan, David and Lim, Hyeontaek and Zoph, Barret and Spiridonov, Alexander and Sepassi, Ryan and Dohan, David and Agrawal, Shivani and Omernick, Mark and Dai, Andrew M. and Pillai, Thanumalayan Sankaranarayana and Pellat, Marie and Lewkowycz, Aitor and Moreira, Erica and Child, Rewon and Polozov, Oleksandr and Lee, Katherine and Zhou, Zongwei and Wang, Xuezhi and Saeta, Brennan and Diaz, Mark and Firat, Orhan and Catasta, Michele and Wei, Jason and Meier-Hellstern, Kathy and Eck, Douglas and Dean, Jeff and Petrov, Slav and Fiedel, Noah},
  title = {PaLM: Scaling Language Modeling with Pathways},
  publisher = {arXiv},
  archivePrefix={arXiv},
  eprint={2204.02311},
  year = {2022},
  primaryClass={cs.CL}
}

@misc{brown2020language,
    title={Language Models are Few-Shot Learners},
    author={Tom B. Brown and Benjamin Mann and Nick Ryder and Melanie Subbiah and Jared Kaplan and Prafulla Dhariwal and Arvind Neelakantan and Pranav Shyam and Girish Sastry and Amanda Askell and Sandhini Agarwal and Ariel Herbert-Voss and Gretchen Krueger and Tom Henighan and Rewon Child and Aditya Ramesh and Daniel M. Ziegler and Jeffrey Wu and Clemens Winter and Christopher Hesse and Mark Chen and Eric Sigler and Mateusz Litwin and Scott Gray and Benjamin Chess and Jack Clark and Christopher Berner and Sam McCandlish and Alec Radford and Ilya Sutskever and Dario Amodei},
    year={2020},
    eprint={2005.14165},
    archivePrefix={arXiv},
    primaryClass={cs.CL}
}

@inproceedings{yu2017seqgan,
  title={Seqgan: Sequence generative adversarial nets with policy gradient},
  author={Yu, Lantao and Zhang, Weinan and Wang, Jun and Yu, Yong},
  booktitle={Thirty-First AAAI Conference on Artificial Intelligence},
  year={2017}
}

@inproceedings{guo2018long,
  title={Long text generation via adversarial training with leaked information},
  author={Guo, Jiaxian and Lu, Sidi and Cai, Han and Zhang, Weinan and Yu, Yong and Wang, Jun},
  booktitle={Thirty-Second AAAI Conference on Artificial Intelligence},
  year={2018}
}

@inproceedings{press2017language,
  title={Language generation with recurrent generative adversarial networks without pre-training},
  author={Press, Ofir and Bar, Amir and Bogin, Ben and Berant, Jonathan and Wolf, Lior},
  booktitle={1st Workshop on Learning to Generate
Natural Language at ICML 2017.},
  year={2017}
}

@inproceedings{tolstikhin2017adagan,
  title={Adagan: Boosting generative models},
  author={Tolstikhin, Ilya O and Gelly, Sylvain and Bousquet, Olivier and Simon-Gabriel, Carl-Johann and Sch{\"o}lkopf, Bernhard},
  booktitle={Advances in Neural Information Processing Systems},
  pages={5424--5433},
  year={2017}
}

@inproceedings{grover2018boosted,
  title={Boosted generative models},
  author={Grover, Aditya and Ermon, Stefano},
  booktitle={Thirty-Second AAAI Conference on Artificial Intelligence},
  year={2018}
}

@article{goodfellow2017nips,
  title={NIPS 2016 Tutorial: Generative Adversarial Networks},
  author={Goodfellow, Ian},
  journal={arXiv preprint arXiv:1701.00160},
  year={2017}
}

@inproceedings{hashimoto2019unifying,
  title={Unifying human and statistical evaluation for natural language generation},
  author={Hashimoto, Tatsunori B and Zhang, Hugh and Liang, Percy},
  booktitle={Proceedings of NAACL-HLT},
  pages={1689--1701},
  year={2019}
}

@misc{caccia2018language,
    title={Language GANs Falling Short},
    author={Massimo Caccia and Lucas Caccia and William Fedus and Hugo Larochelle and Joelle Pineau and Laurent Charlin},
    year={2018},
    eprint={1811.02549},
    archivePrefix={arXiv},
    primaryClass={cs.CL}
}

@inproceedings{desjardins2011tracking,
  title={On tracking the partition function},
  author={Desjardins, Guillaume and Bengio, Yoshua and Courville, Aaron C},
  booktitle={Advances in neural information processing systems},
  pages={2501--2509},
  year={2011}
}

@techreport{rosenfeld1994adaptive,
  title={Adaptive statistical language modeling; a maximum entropy approach},
  author={Rosenfeld, Ronald},
  year={1994},
  institution={Carnegie-Mellon University, Department of Computer Science}
}

@article{khudanpur2000maximum,
  title={Maximum entropy techniques for exploiting syntactic, semantic and collocational dependencies in language modeling},
  author={Khudanpur, Sanjeev and Wu, Jun},
  journal={Computer Speech \& Language},
  volume={14},
  number={4},
  pages={355--372},
  year={2000},
  publisher={Elsevier}
}

@article{mikolov2013efficient,
  title={Efficient estimation of word representations in vector space},
  author={Mikolov, Tomas and Chen, Kai and Corrado, Greg and Dean, Jeffrey},
  journal={arXiv preprint arXiv:1301.3781},
  year={2013}
}

@article{williams1992simple,
  title={Simple statistical gradient-following algorithms for connectionist reinforcement learning},
  author={Williams, Ronald J},
  journal={Machine learning},
  volume={8},
  number={3-4},
  pages={229--256},
  year={1992},
  publisher={Springer}
}

@article{kusner2016gans,
  title={Gans for sequences of discrete elements with the gumbel-softmax distribution},
  author={Kusner, Matt J and Hern{\'a}ndez-Lobato, Jos{\'e} Miguel},
  journal={arXiv preprint arXiv:1611.04051},
  year={2016}
}

@article{che2017maximum,
  title={Maximum-likelihood augmented discrete generative adversarial networks},
  author={Che, Tong and Li, Yanran and Zhang, Ruixiang and Hjelm, R Devon and Li, Wenjie and Song, Yangqiu and Bengio, Yoshua},
  journal={arXiv preprint arXiv:1702.07983},
  year={2017}
}

@article{rajeswar2017adversarial,
  title={Adversarial generation of natural language},
  author={Rajeswar, Sai and Subramanian, Sandeep and Dutil, Francis and Pal, Christopher and Courville, Aaron},
  journal={arXiv preprint arXiv:1705.10929},
  year={2017}
}

@article{radford2018improving,
  title={Improving Language Understanding by Generative Pre-Training},
  author={Radford, Alec and Narasimhan, Karthik and Salimans, Tim and Sutskever, Ilya},
  year={2018}
}

@article{radford2019language,
  title={Language models are unsupervised multitask learners},
  author={Radford, Alec and Wu, Jeffrey and Child, Rewon and Luan, David and Amodei, Dario and Sutskever, Ilya},
  journal={OpenAI Blog},
  volume={1},
  number={8},
  year={2019}
}

@article{rigollet2018entropic,
  title={Entropic optimal transport is maximum-likelihood deconvolution},
  author={Rigollet, Philippe and Weed, Jonathan},
  journal={Comptes Rendus Mathematique},
  volume={356},
  number={11-12},
  pages={1228--1235},
  year={2018},
  publisher={Elsevier}
}

@article{sason2015reverse,
  title={On reverse Pinsker inequalities},
  author={Sason, Igal},
  journal={arXiv preprint arXiv:1503.07118},
  year={2015}
}

@MISC{amodei2018ai,
  title        = "{AI} and Compute",
  booktitle    = "{OpenAI}",
  author       = "Amodei, Dario and Hernandez, Danny",
  year         =  2018,
  howpublished = "\url{https://openai.com/blog/ai-and-compute/}",
  note         = "Accessed: 2020-6-1"
}

@INPROCEEDINGS{strubell2019energy,
  title     = "Energy and Policy Considerations for Deep Learning in {NLP}",
  booktitle = "Proceedings of the 57th Annual Meeting of the Association for
               Computational Linguistics",
  author    = "Strubell, Emma and Ganesh, Ananya and McCallum, Andrew",
  publisher = "aclweb.org",
  pages     = "3645--3650",
  year      =  2019
}




\appendix
\section{Proof of Lemma 3}\label{sec:proof_lemma_seq}

We proceed analogously as in the proof of Lemma~\ref{lem:decreaseLogloss}. We first note that

\begin{align*}
     \hat{L}(q; \cS) = -\hE_{\cS} \left[ \log \prod_{i=1}^{N} q(x_{i} \mid x_1, \dots, x_{i-1}) \right] = -\sum_{i=1}^{N} \hE_{\cS} \log q(x_{i} \mid x_1, \dots, x_{i-1}),
\end{align*}
and
\begin{align*}
    \hat{L}(q'; \cS) &= -\hE_{\cS} \left[ \log \prod_{i=1}^{N} q'(x_{i} \mid x_1, \dots, x_{i-1}) \right] \\&= \sum_{i=1}^{N} -\hE_{\cS} \log q(x_{i} \mid x_1, \dots, x_{i-1}) + b g( x_1, \dots, x_{i}) + \log Z_{q'}(x_1, \dots, x_{i-1})
\end{align*}
Let us use the short-hand notation $x_{1:i} \triangleq (x_1,\dots, x_i)$. 
Subtracting the two equalities above we obtain
\begin{align*}
    \hat{L}(q; \cS) - \hat{L}(q'; \cS) &=  \sum_{i=1}^{N} \hE_{\cS}\left[ -b g( x_{1:i}) - \log Z_{q'}(x_{1:i-1}) \right],
\end{align*}
which, after adding and subtracting $ \hE_{\cS}\E_{w \sim q( w \mid x_{1:i-1})} g(x_{1:i-1},w)$ and rearranging terms, yields
\begin{align}
    \hat{L}(q; \cS) - \hat{L}(q'; \cS) &= b \left[ \sum_{i=1}^{N} \hE_{\cS}  \biggl( \E_{w \sim q( w \mid x_{1:i})}g(x_{1:i-1},w) - g(x_{1:i}) \biggr) \right] \\ & \quad - \sum_{i=1}^{N} \hE_{\cS} \left[ \log Z_{q'}(x_{1:i-1}) -b \E_{w \sim q( w \mid x_{1:i-1})}  g(x_{1:i-1},w) \right] \\
    &= b N\hat{\beta}(g) - \sum_{i=1}^{N} \hE_{\cS} \left[ b\E_w g(x_{1:i-1},w) +  \log Z_{q'}(x_{1:i-1}) \right] \label{eq:last_step_diff}
\end{align}
By assumption we have $Nb\hat\beta(g)\geq Nb^2$, so it it remains to show that the second term is upper bounded by $Nb^2/2$. Using, as before, the bound $\log r\leq r-1$ for every $r=Z_{q'}(x_{1:i-1}) \geq 0$, 
we get that, for every $i=1,\dots,N$:
\begin{align*}
 \hE_{\cS} \left[ b\E_w g(x_{1:i-1},w) +  \log Z_{q'}(x_{1:i-1}) \right] &\leq 
  \hE_{\cS} \left[ b\E_w g(x_{1:i-1},w) +  Z_{q'}(x_{1:i-1}) - 1\right] \\
  &=  \hE_{\cS} \left[ b\E_w g(x_{1:i-1},w) + \E_w e^{-bg(x_{1:i-1},w)} -1  \right] \\
  &= \hE_{\cS}\E_w  \left[ bg(x_{1:i-1},w) +  e^{-bg(x_{1:i-1},w)} -1  \right] \\
  &\leq \hE_{\cS}\E_w  \left(bg(x_{1:i-1},w)/2 \right)^2 \leq \frac{b^2}{2}
\end{align*}
where the last inequality follows again from the fact that $g(x) \in [0,1]$ for any $x$. Therefore, the sum over these $N$ terms is upper bounded by $N\frac{b^2}{2}$, which combined with \eqref{eq:last_step_diff}, yields the desired result.

\end{document}